\documentclass{amsart}
\usepackage{graphicx}
\usepackage{amssymb}
\usepackage{multirow}
\newtheorem{thm}{Theorem}[section]

\newtheorem{lem}[thm]{Lemma}

\theoremstyle{definition}
\newtheorem{defn}[thm]{Definition}
\theoremstyle{remark}

\numberwithin{equation}{section}

\begin{document}

\title[rule-weighted vs terminal-weighted context-free grammars]{Rule-weighted and terminal-weighted context-free grammars have identical expressivity}%
\author{Yann Ponty}%
\address{CNRS/LIX -- Ecole Polytechnique -- Palaiseau, France\\
AMIB Project -- INRIA -- Saclay, France\\
}%
\email{yann.ponty@lix.polytechnique.fr}%


\newcommand{\ProbT}{p_{\PondT,n}}
\newcommand{\PondT}{\pi}
\newcommand{\ProbR}{p_{\PondR,n}}
\newcommand{\PondR}{\lambda}
\newcommand{\PondRA}{\lambda^{\blacktriangleright}}
\newcommand{\PondRB}{\lambda^{\blacktriangleleft}}
\newcommand{\Eqdef}{:=}
\newcommand{\GramWT}{\mathcal{G}_\PondT}
\newcommand{\GramWR}{\mathcal{G}_\PondR}
\newcommand{\GramWRG}{\mathcal{H}_{\PondR'}}
\newcommand{\R}{\mathcal{R}}
\newcommand{\Gram}{\mathcal{G}}
\newcommand{\NtSet}{\mathcal{N}}
\newcommand{\Nt}{N}
\newcommand{\T}{t}
\newcommand{\Voc}{\Sigma}
\newcommand{\ProdRules}{\mathcal{P}}
\newcommand{\Axiom}{\mathcal{S}}
\newcommand{\Def}[1]{\textbf{#1}}
\newcommand{\Production}{\rightarrow}
\newcommand{\Lang}[1]{\mathcal{L}(#1)}

\begin{abstract}
Two formalisms have recently been proposed to perform a non-uniform random generation of combinatorial objects based on context-free grammars. The
former, introduced by Denise \emph{et al}, associates weights with letters, while the latter, recently explored by Weinberg~\emph{et al} in the context of random generation, associates weights to transitions. In this short note, we use a trivial modification of the Greibach Normal Form transformation algorithm, due to Blum and Koch, to show the equivalent expressivities of these two formalisms.
\end{abstract}
\maketitle
\section{Introduction}

The random generation of combinatorial objects is one of the natural applications of enumerative combinatorics. 
Following general principles outlined by Wilf~\cite{wilf77}, Flajolet~\emph{et al}~\cite{flajoletcalculus} proposed a fully-automated algebraic approach 
for the extensive class of decomposable combinatorial objects, a large class of objects that includes context-free languages. This pioneering work was later 
completed by the introduction of Boltzmann samplers, an alternative family of random generation algorithms based on analytical properties of the underlying generating functions~\cite{DuFlLoSc02}.
However, these works only addressed the uniform distribution, while many applications of random generation (e.g. in RNA bioinformatics~\cite{DiLa03}) require
non-uniform distributions to be modeled.

To that purpose, Denise~\emph{et al} \cite{Denise2000} introduced (terminal)-weighted grammars, a non-uniform framework where the terminal symbols (letters)
are associated with a real positive value, inherited multiplicatively by words in the language.
Such weights were then used, through a trivial renormalization, to induce a probability distribution on the finite set of words of a given length.
Generic random generation algorithms were proposed~\cite{Denise2000} and implemented within a general random generation toolbox~\cite{PoTeDe06}.
Analytic and numerical approaches were proposed for figuring out suitable set of weights that would mimic a given, observed, distribution~\cite{Denise2010}.
Finally, a multidimensional rejection scheme was explored to sample words of a given composition, yielding efficient algorithms by generalizing the 
principles of Boltzmann sampling~\cite{Bodini2010}.

More recently, Weinberg~\emph{et al}~\cite{Weinberg2010} proposed an alternative definition for weighted grammars, associating positive 
real-values to rules instead of terminal letters. The authors proposed a random generation procedure based on formal grammar manipulations, followed by a
call to an unranking algorithm due to Martinez and Molinero~\cite{Martinez00ageneric}. However, the relative expressivities, in term of the distribution 
induced by the respective weighting schemes, of the two formalisms were not compared.

In this short note, we establish the equivalence of the two formalisms with respect to their induced distributions. After this short introduction, 
we remind in Section~\ref{sec:def} the definitions of terminal-weighted and rule-weighted grammars. Then we turn to an analysis of the relative 
expressivities of the two formalisms, and establish in Section~\ref{sec:TtoW} that any terminal-weighted grammar can be \emph{simulated} by a 
rule-weighted grammar. Furthermore, we use a Greibach Normal Form transformation to prove, in Section~\ref{sec:WtoT}, that any rule-weighted grammar
can be transformed into a terminal-weighted grammar inducing the same probability distribution, from which one concludes on the equivalence of the two formalisms.
We conclude in Section~\ref{sec:conclusion} with some closing remarks and perspectives.

\section{Definitions}\label{sec:def}
	A \Def{context-free grammar} is a 4-tuple $\Gram=(\Voc,\NtSet,\ProdRules,\Axiom)$ where
	\begin{itemize}
		\item $\Voc$ is the alphabet, i.e. a finite set of terminal symbols, also called \Def{letters}.
		\item $\NtSet$ is a finite set of \Def{non-terminal symbols}.
		\item $\ProdRules$ is the finite set of production \Def{rules} of the form $\Nt\Production X$,
		where $\Nt\in\NtSet$ is a non-terminal and $X\in \{\Voc\,\cup\,\NtSet\}^*$ is a sequence of letters and non-terminals.
		\item $\Axiom$ is the \Def{axiom} of the grammar, i. e. the initial non-terminal.
	\end{itemize}
We will denote by $\Lang{\Gram}_n$ the set of all words of length $n$ generated by $\Gram$.
This set is generated by iteratively applying production rules to non-terminals until a word in $\Voc^*$ is obtained.

Note that the non-terminals on the right-hand side of a production rule can be independently derived. It follows that the derivation process, starting from the
initial axiom and ending with a word $w$ over the terminal alphabet, can be represented by a \Def{parse tree} $d_w$. This (ordered directed) tree associates production rules to each internal
node and terminal letters to each leaf, such that the $i$-th child of a node labeled with $\Nt \to x_1.\cdots.x_k$ is either a terminal letter $x_i\in \Voc$ or a further derivation of $x_i\in \ProdRules$, starting from a root node that derives the axiom $\Axiom$.

{\bf Assumptions: }Let us assume, for the sake of simplicity, that the grammars considered in the following are \Def{unambiguous}, i.e. that any word in $\Lang{\Gram}_n$
has exactly one associated parse tree. Moreover, let us assume, without loss of generality, that the grammar is given using a \Def{binary variant} of the \Def{Chomsky Normal Form (CNF)},
which partitions the non-terminals into four classes, restricting their production rules to:
\begin{itemize}
  \item {\bf Axiom:} $\Axiom \to \Nt$, $\Nt\in\NtSet$, and/or $\Axiom \to \varepsilon$.
  \item {\bf Unions:} $\Nt \to \Nt'\;|\;\Nt''$, such that $\Nt,\Nt',\Nt''\in\NtSet/\{\Axiom\}$.
  \item {\bf Products:} $\Nt \to \Nt'.\Nt''$, such that $\Nt,\Nt',\Nt''\in\NtSet/\{\Axiom\}$.
  \item {\bf Terminals:} $\Nt \to t, \;t\in\Voc$.
\end{itemize}
Finally, we will postulate the absence of non-productive terminals, e.g. having rules of the form $\Nt \to \Nt.\Nt'$.

\subsection{Terminal-Weighted Grammars}
A non-uniform distribution can be postulated on the language generated by the grammar. To that purpose, two formalisms have been independently proposed, reminded here for the sake of completeness.
	\begin{defn}[(Terminal)-Weighted Grammar~\cite{Denise2000}]
		A {terminal-weighted grammar} $\GramWT$ is a 5-tuple $\GramWT=(\PondT,\Voc,\NtSet,\ProdRules,\Axiom)$ where:
\begin{itemize}
		\item $(\Voc, \NtSet, \ProdRules, \Axiom)$ defines a context-free grammar,
		\item $\PondT:\Voc\to\mathbb{R}^{+}$ is a \Def{terminal-weighting function} that associates a non-null positive real-valued weight $\PondT_\T$ to each
		terminal symbol $\T$.
\end{itemize}
The weight of a word $w\in\Lang{\GramWT}$ is then given by $$\PondT(w)=\prod_{\T\in \Voc} \PondT_{\T}^{|w|_{\T}}$$ and extended into a
probability distribution over $\Lang{\Gram}_n$ by
$$ \ProbT(w) = \frac{\PondT(w)}{\sum_{w\in \Lang{\Gram}_n} \PondT(w)}.$$
	\end{defn}

\subsection{Rule-Weighted Grammars}
	\begin{defn}[(Rule)-Weighted Grammar~\cite{Weinberg2010}]
		A {rule-weighted grammar} $\GramWR$ is a 5-tuple $\GramWR=(\PondR,\Voc,\NtSet,\ProdRules,\Axiom)$ where:
\begin{itemize}
		\item $(\Voc, \NtSet, \ProdRules, \Axiom)$ defines a context-free grammar,
		\item $\PondR:\ProdRules\to\mathbb{R}^{+}$ is a \Def{rule-weighting function} that associates a positive non-null real-valued\footnote{More precisely, Weinberg~\emph{et al} restrict their formalism to rational weights, based on the rationale that real-numbers would lead to unstable computations. However their framework could easily be extended to any computable real numbers without loss of precision, e.g. by implementing a \emph{confidence intervals} approach described in Denise and Zimmermann~\cite{DeZi99}, therefore we consider a trivial extension of this formalism here.} weight $\PondR_r$ to each derivation $r\in\ProdRules$, using the notation $\Nt \to_{y} X$ to indicate the association of a weight $\PondR_r = y$ to a rule $r=(\Nt \to X)$.
\end{itemize}
The weight function $\PondR$ can then be extended multiplicatively over $\Lang{\GramWR}$ through
$$\PondR(w)=\prod_{\substack{r\in d_w\\ r =(\Nt \to_{\PondR_{r}} X)}} \PondR_{r},\; \forall w \in \Lang{\GramWR}$$
where $d_w$ is the (unique) parse tree of $w$ in $\GramWR$.
This induces a probability distribution over $\Lang{\GramWR}_n$ such that
$$ \ProbR(w) = \frac{\PondR(w)}{\sum_{w\in \Lang{\GramWR}_n} \PondR(w)}.$$
	\end{defn}

\section{Any terminal-weighted distribution can be obtained using a rule-weighted grammar}\label{sec:TtoW}
\begin{thm}
  For any terminal-weighted grammar $\GramWT$, there exists a rule-weighted grammar $\GramWR$, $\Lang{\GramWT}=\Lang{\GramWR}$, inducing an identical probability distribution.
\end{thm}
\begin{proof}
  We give a constructive proof of this theorem. For any grammar $\GramWT=(\PondT,\Voc,\NtSet,\ProdRules,\Axiom)$, let us
  consider the rule-weighted grammar defined by $\GramWR:=(\PondR,\Voc,\NtSet,\ProdRules,\Axiom)$, such that $\PondR(\Nt \to \T)=\PondT_\T$ and $\PondR(\cdot)=1$ otherwise.

  Clearly, the production rules and axioms of $\GramWR$ and $\GramWT$ are identical, therefore one has $\Lang{\GramWR} = \Lang{\GramWT}$ and, in particular, $$\Lang{\GramWR}_n = \Lang{\GramWT}_n,\; \forall n\ge 0.$$

  Let us now remark that any terminal letter $\T$ in a produced word $w$ results from the application of a rule of the form $\Nt \to \T$,
and that the parse trees in $\GramWT$ and $\GramWR$ of any word $w\in \Lang{\GramWT}=\Lang{\GramWR}$ are identical.
  It follows that the occurrences of the terminal letter $\T$ in $w$ are in bijection with the occurrences of the $\Nt \to \T$ rule in its parse tree $d_w$,
  and therefore $$\PondR(w)=\prod_{r =(\Nt \to_{\PondR_{r}} \T)\in d_w} \PondR_{r} = \prod_{(\Nt \to_{\PondR_{r}} \T)\in d_w} \PondT_{t} = \PondT(w),\; \forall w\in \Lang{\GramWT}.$$
  Since $\Lang{\GramWR}_n = \Lang{\GramWT}_n$, then one has $$\sum_{w\in \Lang{\GramWR}_n} \PondR(w) = \sum_{w\in \Lang{\GramWT}_n} \PondT(w),$$ and we conclude that, for any length $n\ge 0$, one has
  $$ \ProbT(w) = \frac{\PondR(w)}{\sum_{w\in \Lang{\GramWR}_n} \PondR(w)} = \frac{\PondT(w)}{\sum_{w\in \Lang{\GramWT}_n} \PondT(w)} = \ProbR(w)$$
  which proves our claim.
\end{proof}
\section{Any rule-weighted distribution can be obtained using a terminal-weighted grammar}\label{sec:WtoT}
\begin{thm}
  For any rule-weighted grammar $\GramWR$, there exists a terminal-weighted grammar $\GramWT$, $\Lang{\GramWR}=\Lang{\GramWT}$, inducing an identical probability distribution.
\end{thm}
\begin{proof}
  Let us first remind the definition of the \Def{Greibach Normal Form} (GNF), which requires each production rule to be of the form:
\begin{itemize}
  \item $\Axiom \to \varepsilon$, where $S$ is the axiom,
  \item $\Nt \to \T.X$, where $\T\in \Voc$ and $X\in\{\Voc\cup \NtSet/\{\Axiom\}\}^*$.
\end{itemize}

  Based on Lemma~\ref{lem:tech} proven below, we know that any rule-weighted grammar in Chomsky-Normal Form can be transformed into a GNF grammar that generates
  the same language and induces the same distribution.  Let us then assume, without loss of generality, that the input grammar $\GramWR=(\PondR,\Voc,\NtSet,\ProdRules,\Axiom)$ 
  is in GNF.
  
  By duplicating the vocabulary, one easily builds a terminal-weighted grammar that
  induces the same probability distribution as $\GramWR$. Namely, let us define $\GramWT = (\PondT,\Voc_{\PondT},\NtSet,\ProdRules_{\PondT},\Axiom)$ such that $\Voc_{\PondT} := \{t_r\}_{r\in \ProdRules}$, $\PondT(t_r) := \PondR(r)$, and
  $$\ProdRules_{\PondT} := \{\Nt \to \T_r.X\;|\; r=(\Nt \to_x \T.X) \in \ProdRules\} \cup \{S \to \varepsilon\; |\; S \to_x \varepsilon \in \ProdRules\}.$$
    Clearly, each terminal letter in a word produced by $\GramWT$ can be unambiguously associated with a rule of $\GramWR$, therefore the weight of any non-empty word is 
    preserved. Furthermore, the generated languages of $\GramWR$ and $\GramWT$ are identical, so the distribution is preserved. Finally, the weight of the empty word $\varepsilon$, implicitly set to $1$ in the new grammar, may generally differ from its original value $\PondR(S \to \varepsilon)$ in $\GramWR$. However, $\varepsilon$ is the only word of length $0$, and therefore has probability $1$ in both grammars. We conclude that the probability distribution induced by $\GramWR$ is the same as that of $\GramWT$.

  \begin{lem}\label{lem:tech}
    For any rule-weighted grammar $\GramWR=(\PondR,\Voc,\NtSet,\ProdRules,\Axiom)$, there exists a grammar $\GramWRG$ in Greibach Normal Form inducing the same distribution.
  \end{lem}
  \begin{proof}
  Again, we use a constructive proof, showing that the weight distribution can be preserved during the transformation of the grammar performed by the Blum and Koch normalisation algorithm~\cite{Blum1999}. Let us state the algorithm:
  \begin{enumerate}
    \item Renumber non-terminals in any order, starting with the Axiom $\Axiom \Rightarrow \Nt_1$.
    \item For $k=1$ to $|\NtSet|$, consider the non-terminal $\Nt_k$:\label{step:b}
    \begin{enumerate}
    \item For each $r=(\Nt_k \to_x \Nt_j.X) \in \ProdRules$, such that $j<k$ and
     $$\Nt_j \to_{x_1} X_1, \Nt_j \to_{x_2} X_2, \cdots, \Nt_j\to_{x_m} X_m,$$
    replace $r$ in $\ProdRules$ as follows
    \begin{equation*}
    \begin{array}{rcl|rcl}
      \multicolumn{3}{c|}{\text{Former rule(s)}}&\multicolumn{3}{c}{\text{New rule(s)}} \\ \hline \hline
      \multirow{4}{*}{$\Nt_k$}&\multirow{4}{*}{$\to_x$}& \multirow{4}{*}{$\Nt_j.X$} & \Nt_k &\to_{x\cdot x_1}& X_1.X\\
      &&& \Nt_k &\to_{x\cdot x_2}& X_2.X\\
      &&& \multicolumn{3}{|c}{\vdots}\\
      &&& \Nt_k&\to_{x\cdot x_m}& X_m.X.
    \end{array}
    \end{equation*}
    \item Fix any left-recursive non-terminal $\Nt_k$ by replacing its rules as follows, using an alternative \emph{chain-rule} construct:
    \begin{equation*}
    \begin{array}{rcl|rcl}
      \multicolumn{3}{c|}{\text{Former rule(s)}}&\multicolumn{3}{c}{\text{New rule(s)}} \\ \hline \hline
       \multirow{2}{*}{$\Nt_k$} &\multirow{2}{*}{$\to_{x_1}$}& \multirow{2}{*}{$\Nt_k.X_1$} & \Nt'_k &\to_{x_1}& X_1.\Nt'_k\\
            &         &           & \Nt'_k &\to_{x_1}& X_1\\
       \multicolumn{3}{c}{\vdots}& \multicolumn{3}{|c}{\vdots}\\
       \multirow{2}{*}{$\Nt_k$} &\multirow{2}{*}{$\to_{x_m}$}& \multirow{2}{*}{$\Nt_k.X_m$} & \Nt'_k &\to_{x_m}& X_m.\Nt'_k\\
            &         &           & \Nt'_k &\to_{x_m}& X_m\\ \hline
       \multirow{2}{*}{$\Nt_k$} &\multirow{2}{*}{$\to_{y_1}$}& \multirow{2}{*}{$Y_1$} & \Nt'_k &\to_{y_1}& Y_1.\Nt'_k\\
            &         &           & \Nt'_k &\to_{y_1}& Y_1\\
       \multicolumn{3}{c}{\vdots}& \multicolumn{3}{|c}{\vdots}\\
       \multirow{2}{*}{$\Nt_k$} &\multirow{2}{*}{$\to_{y_{m'}}$}& \multirow{2}{*}{$Y_{m'}$} & \Nt'_k &\to_{y_{m'}}& Y_m.\Nt'_k\\
            &         &           & \Nt'_k &\to_{y_{m'}}& Y_{m'}\\ \hline
    \end{array}
    \end{equation*}

    \end{enumerate}
    \item For $k=|\NtSet|$ down to $1$, consider the non-terminal $\Nt_k$:\label{step:c}
    \begin{enumerate}
      \item For each $r=(\Nt_k \to_x \Nt_j.X) \in \ProdRules$, such that $j>k$ and
      $$\Nt_j \to_{x_1} X_1, \Nt_j \to_{x_2} X_2, \cdots, \Nt_j\to_{x_m} X_m,$$
    replace $r$ in $\ProdRules$ as follows
    \begin{equation*}
    \begin{array}{rcl|rcl}
      \multicolumn{3}{c|}{\text{Former rule(s)}}&\multicolumn{3}{c}{\text{New rule(s)}} \\ \hline \hline
      \multirow{4}{*}{$\Nt_k$}&\multirow{4}{*}{$\to_x$}& \multirow{4}{*}{$\Nt_j.X$} & \Nt_k &\to_{x\cdot x_1}& X_1.X\\
      &&& \Nt_k &\to_{x\cdot x_2}& X_2.X\\
      &&& \multicolumn{3}{|c}{\vdots}\\
      &&& \Nt_k&\to_{x\cdot x_m}& X_m.X.
    \end{array}
    \end{equation*}
    \end{enumerate}
  \end{enumerate}
  One easily verifies that, after any iteration of step~(\ref{step:b}), the grammar no longer contains any rule $\Nt_j \to \Nt_l.X$ such that $l\le j\le k$.
  This holds for $N_k$ which, after the full execution of step~(\ref{step:b}), does not depend from any non-terminal, and is therefore in GNF.
  Furthermore, one may assume that, anytime a non-terminal $N_k$ is considered during step~(\ref{step:c}), every $\Nt_{j}$ such that $k<j$ is in GNF.
  Consequently, the expansion of $N_j$ only creates rules that are GNF-compliant, thus $\Nt_k$ is in GNF at the end of the iteration.

  Let us denote by $\GramWRG=(\PondR',\Voc,\NtSet',\ProdRules',\Axiom)$ the rule-weighted grammar obtained at the end of the execution. One first remarks that both the expansions (Steps~(\ref{step:b}.a) and (\ref{step:c}.a))
  and the chain-rule reversal (Steps~(\ref{step:b}.b)) preserve the generated language, so that the language generated by a non-terminal in $\GramWR$ is also the language generated by its corresponding non-terminal in $\GramWRG$. Furthermore, one can prove, by induction on the number of derivations
  required to generate a word, that the induced probability distribution is kept invariant by rules substitutions operated by the algorithm.

  To that purpose, let us first extend the definition of a rule-weighting function to include a partial derivation instead of a single non-terminal. Namely, $\PondR_{X}(w)$ will represent the weight of $w$, as derived from $X\in \{\ProdRules\cup\Voc\}^*$ and, in particular, one has $\PondR_{\Axiom} \equiv \PondR$.
  Let us now consider the rule-weighting functions $\PondRA$ and $\PondRB$, respectively induced by the grammar before and after a modification:
    \begin{itemize}
      \item {\it Induction hypothesis:} Any word $w$ generated from any $X \{\ProdRules\cup\Voc\}^*$ using $d$ derivations, $1\le d< n$, is such that $\PondR_{X}'(w) = \PondR_{X}''(w)$.
      \item {\it Rule expansion (Steps~(\ref{step:b}.a) and (\ref{step:c}.a)):} Consider a word $w$, generated using $n$ derivations from some non-terminal $\Nt$.
      Clearly, if $\Nt\neq \Nt_k$ or if the first derivation used is $\Nt\to X' \neq \Nt_j.X$, then the rule used to generated $w$ is not affected by the modification. The induction hypothesis applies and one trivially gets $\PondRA_{N_k}(w) = \PondRB_{N_k}(w)$.

       Consider the initial state of the grammar. When $w$ results from a derivation $\Nt\to_x \Nt_j.X$, then there exists a (unique) decomposition $w=w'.w''$, where $w'$ is produced by the application of some rule $\Nt_j \to_{x_i} X_i$, and $w''$ is derived from $X$. The weight of $w$ is then given by
       $\PondRA_{N_k}(w)=x\cdot x_i\cdot \PondRA_{X_i}(w')\cdot \PondRA_{X}(w'')$.

       In the modified version of the grammar, $w=w'.w''$ unambiguous derives from an application of the new rule $\Nt_k \to_{x\cdot x_i} X_i.X$ ($w'\in\Lang{X_i}$ and $w''\in\Lang{X}$), with associated weight $\PondRB_k(w)=x\cdot x_i\cdot \PondR_k''(w')\cdot \PondRB_k(w'')$. Since both $w'$ and $w''$ are generated using less than $n$ derivations, then the induction hypothesis applies and one gets
       $$\PondRB_{N_k}(w)=x\cdot x_i\cdot \PondRB_{X_i}(w')\cdot \PondRB_{X}(w'') = x\cdot x_i\cdot \PondRA_{X_i}(w')\cdot \PondRA_{X}(w'') = \PondRA_{N_k}(w).$$

       \item {\it Chain-rule reversal (Steps~(\ref{step:b}.b)):} Any word produced using the initial \emph{left-recursive} chain-rule can be uniquely decomposed as $w=w'.w''_1.\cdots.w''_p$,
       where  $w'$ is generated from some $\Nt_k \to_{y_q} Y_{q}, q\in[1,m'],$ and each $w''_i$ is generated by some rule $\Nt_k \to_{x_{q_i}}\Nt_k.X_{q_i}, q_i\in[1,m]$. Its weight is therefore given by $\PondRA_{\Nt_k}(w)= y_q\cdot \left(\prod_{i=1}^m x_{q_i}\right)\cdot \PondRA_{{Y_{q}}}(w')\cdot \left(\prod_{i=1}^m \PondRA_{{X_{q_i}}}(w''_i)\right)$.

       After chain-rule reversal, the same decomposition $w=w'.w''_1.\cdots.w''_p$ holds, but the sequence of derivation is now either $\Nt_k \to_{y_q} Y_q$ ($w=w'$),
       or \begin{align*}\Nt_k &\to_{y_q} Y_q.\Nt'_k \to_{x_{q_1}} Y_q.X_{q_1}.\Nt'_k \\
       &\rightsquigarrow Y_q.X_{q_1}.\cdots.X_{q_{m-1}}.\Nt'_k \\
       & \to_{x_{q_m}} Y_q.X_{q_1}.\cdots.X_{q_{m}} \rightsquigarrow w'.w''_1.\cdots.w''_p.\end{align*}
       In both cases, the induction hypothesis applies for each element of the decomposition, and the weight of $w$ in the new decomposition is given by
       \begin{align*}
       \PondRB_{\Nt_k}(w)&= y_q\cdot \left(\prod_{i=1}^m x_{q_i}\right)\cdot \PondRB_{{Y_{q}}}(w')\cdot \left(\prod_{i=1}^m \PondRB_{{X_{q_i}}}(w''_i)\right)\\
                         &= y_q\cdot \left(\prod_{i=1}^m x_{q_i}\right)\cdot \PondRA_{{Y_{q}}}(w')\cdot \left(\prod_{i=1}^m \PondRA_{{X_{q_i}}}(w''_i)\right) = \PondRA_{\Nt_k}(w).
       \end{align*}
    \end{itemize}
    It follows that the weight of any word is left unchanged by the substitutions performed in the algorithm. Since the generated language is also preserved, then such a preservation of the weights implies a preservation of the probabilities. We conclude that the returned grammar, in addition to being in GNF, also induces the same probability distribution as $\GramWR$.
  \end{proof}
\end{proof}
\section{Conclusion}\label{sec:conclusion}
Using a trivial modification of the Blum and Koch algorithm~\cite{Blum1999}, we showed that weighting terminal or weighting rules have equal expressive power, i.e. that any distribution captured by the former formalism is also captured by the other and vice-versa.

While both proofs are relatively trivial, going from rule-weighted grammars to terminal-weighted grammars turned out to be more involved than the alternative, leading to an increase of the number of rules. However, this observation might be deceptive, as the choice of the Greibach Normal Form as an intermediate form is only one out of possibly many alternatives, and one could devise more efficient grammar transforms capturing the same distributions. Moreover, it is noteworthy that, even if one chooses to use GNF grammars, there still seems to be a gap between the $O(|\ProdRules|^4)$ size of the grammar returned by the Blum and Koch algorithm, and the minimum $O(|\ProdRules|^2)$ increase observed for some infinite family of grammars, motivating the search for better GNF transformation algorithms.

\bibliographystyle{amsplain}
\bibliography{biblio}
\end{document}